%% file: similarity4.tex
\documentclass[12pt,a4paper]{amsart}

\usepackage{amsmath,mathrsfs}
\usepackage{amssymb}
\usepackage{color}
\usepackage{dsfont}
\usepackage[latin1]{inputenc}

\usepackage{float}

\usepackage{graphicx}
\usepackage{hyperref}

\newcommand\NoBlackBoxes{\global\overfullrule0pt}
\NoBlackBoxes

\makeatletter
\let\serieslogo@\relax
\let\@setcopyright\relax

\newtheorem{definition}{Definition}[section]
\newtheorem{theorem}[definition]{Theorem}

\newtheorem{rem}[definition]{Remark}

\newtheorem{corollary}[definition]{Corollary}

\renewcommand{\P}{{\mathbb{P}}}
\newcommand{\E}{{\mathbb{E}}}

\newcommand{\X}{{\mathcal{X}}}

\renewcommand{\epsilon}{\varepsilon}
\renewcommand{\phi}{\varphi}

\newcommand{\V}{{\mathbb{V}}}

\newcommand{\tc}{{\tilde c}}

\begin{document}

\title{Associative Memories to Accelerate Approximate Nearest Neighbor Search}

\author[Vincent Gripon]{Vincent Gripon}
 \address[Vincent Gripon]{Telecom Bretagne UMR CNRS Lab-STICC}
 \email[Vincent Gripon]{vincent.gripon@telecom-bretagne.eu}

 \author[Matthias L\"owe]{Matthias L\"owe}
 \address[Matthias L\"owe]{Fachbereich Mathematik und Informatik,
University of M\"unster,
 Einsteinstra\ss e 62,
48149 M\"unster,
Germany}
 \email[Matthias L\"owe]{maloewe@math.uni-muenster.de}

 \author[Franck Vermet]{Franck Vermet}
 \address[Franck Vermet]{Laboratoire de Math\'ematiques, UMR CNRS 6205, Universit\'e de Bretagne Occidentale,  6, avenue Victor Le Gorgeu\\
 CS 93837\\
 F-29238 BREST Cedex 3\\
 France}

 \email[Franck Vermet]{Franck.Vermet@univ-brest.fr}

\begin{abstract}
Nearest neighbor search is a very active field in machine learning for it appears in many application cases, including classification and object retrieval. In its canonical version, the complexity of the search is linear with both the dimension and the cardinal of the collection of vectors the search is performed into. Recently many works have focused on reducing the dimension of vectors using quantization techniques or hashing, while providing an approximate result. In this paper we focus instead on tackling the cardinal of the collection of vectors. Namely, we introduce a technique that partitions the collection of vectors and stores each part in its own associative memory. When a query vector is given to the system, associative memories are polled to identify which one contain the closest match. Then an exhaustive search is conducted only on the part of vectors stored in the selected associative memory. We study the effectiveness of the system when messages to store are generated from i.i.d. uniform $\pm$1 random variables or 0-1 sparse i.i.d. random variables. We also conduct experiment on both synthetic data and real data and show it is possible to achieve interesting trade-offs between complexity and accuracy.
\end{abstract}


\maketitle

\newcommand{\vep}{\varepsilon}

\section{Introduction}
Nearest neighbor search is a fundamental problem in computer science that consist of finding the data point, in a previously given collection, that is the closest to some query one, according to a specific metric or similarity measure. When points are distributed uniformly and independently in the search space, and when the space is much larger than the cardinality of data points, it is well known that one cannot expect to solve this problem with a complexity lesser than linear with cardinality and dimensionality of data points. This complexity is achieved by using an exhaustive search where distances between all data points and the query one are computed.

Nearest neighbor search is found in a wide set of applications, including object retrieval, classification, computational statistics, pattern recognition, vision, data mining\dots In many of these domains, collections and dimensions are large, leading to unreasonable computation time when using exhaustive search. For this reason, many recent works have been focusing on approximate nearest neighbor search, where complexities can be greatly reduced at the cost of a nonzero error probability in retrieving the closest match. Most methods act either on cardinality~\cite{muja_lowe,muja_lowe_2009,gong2011iterative} or on dimensionality~\cite{jegou2011product,datar2004locality,gripon_similarity} of data points. Reducing cardinality often requires to be able to partition space in order to identify interesting regions that are likely to contain the nearest neighbor, thus reducing the number of distances to compute, whereas reducing dimensionality is often achieved through quantization techniques, and thus approximate distance computation.

In this note we will present and investigate an alternative approach that was introduced in \cite{yu2015neural} in the context of sparse data. Its key idea is to partition the patterns into equal-sized classes and to compute the overlap of an entire class with the given query vector using associative memories based on neural networks. If this overlap of a given class is above a certain threshold, we decide that the given vector is similar to one of the vectors in the considered class and we perform exhaustive search on this class, only. Obviously, this is a probabilistic method that comes with a certain probability of failure. We will impose two restrictions on the relation between dimension and the size of the database, or equivalently, between dimension and the number of vectors in each class: On the one hand, we want the size of the database to be so large that the algorithm is efficient compared to other methods (primarily exhaustive search), on the other hand we need it to be small enough, to let the error probability converge to 0, as dimension (and possibly the size of the database) converge to infinity.

More precisely, in this note we will consider two principally different scenarios: one where the input patterns are sparse, i.e. consist of $0$s and $1$s with a majority of $0$s and another scenario, where the data are ``unbiased''. This situation is best modeled by choosing the coordinates of all the input vector as independent and identically distributed (i.i.d. for short) $\pm 1$-random variables with equal probabilities for $+1$ and for $-1$. Even though the principal method is similar, both cases have their individual difficulties. As a matter of fact, similar problems occur when analyzing the Hopfield model \cite{Hopfield1982} of neural networks for unbiased patterns (see \cite{MPRV}) or for sparse, biased patterns (see e.g. \cite{GHLV16} or \cite{Lo_biased}). We will analyze the situation of sparse patterns in Section 3, while Section 4 is devoted to the investigation of the situation where the patterns are unbiased and dense. In Section 5 we will support our findings with simulations on classical approximate nearest neighbor search benchmarks.

\section{Related work}

Part of the literature focus on using tree-based methods to partition space~\cite{muja_lowe_2009}, resulting with difficulties when facing high dimension spaces~\cite{muja_lowe}. When facing real valued vectors, many techniques propose binary encoding of vectors~\cite{gong2011iterative,he2013k,weiss2009spectral}, leading to very efficient reductions in time complexity. As a matter of fact, search in Hamming space can be performed very efficiently, but it is well known that premature discretization of data usually leads to significant loss in precision.

In order to achieve better performance, quantization techniques based on Product Quantization (PQ) and Locality Sensitive Hashing (LSH) can be used. In PQ~\cite{jegou2011product}, vectors are split into subvectors. Each subvector space is quantized independently from the others. Then a vector is quantized by concatenating the quantized versions of its subvectors. Optimized versions of PQ have been proposed by performing joint optimization~\cite{ge2013optimized,norouzi2013cartesian}. In LSH, multiple hash functions are defined, resulting in storing each point index multiple times~\cite{norouzi2012fast,datar2004locality}.

In~\cite{yu2015neural}, the authors propose to split data points into multiple parts, each one stored in its own sparse associative memory. The ability of these memories to store a large number of messages result in very efficient reduction in time complexity when data points are sparse binary vectors. In~\cite{gripon_similarity}, the authors propose to use sum of vectors instead of associative memories and discuss optimization strategies (including online scenarios). The methods we analyze in this note are in the same vein.

\section{Search Problems with sparse patterns}
\label{sparse}

In this section we will treat the situation, where we try to find a pattern that is closest to some input pattern and all the patterns are binary and sparse. To be more precise, assume that we are given n sparse vectors or patterns
$x^\mu, \mu=1, \ldots, n$ with $x^\mu \in \{0,1\}^d$ for some (large) $d$. Assume that the $(x^\mu)$ are random and i.i.d. and have i.i.d. coordinates such that
$$
\P( x_i^\mu= 1)= \frac cd = 1- \P(x_i^\mu=0)
$$
for all $i= 1, \ldots ,  d$ and all $\mu=1, \ldots ,n$.
To describe a sparse situation we will also assume $c$ may depend on $d$ but that $c/d \to 0$ as  $c$ and $d$ become large. Actually we will even need that this convergence is fast enough.  Moreover, we will always assume that there is an $x^\mu$ such that the query pattern $x^0$ has a macroscopically large overlap with $x^\mu$, hence that $\sum_{l=1}^d x_l^0 x_l^{\mu}$ is of order $c$.

We will actually begin with the situation that $x^0=x^\mu$ for an index $\mu$. The situation where $x^0$ is a perturbed version of $x^\mu$ is then similar.

The algorithm proposed in \cite{yu2015neural} now proposes to partition the patterns into equal-sized classes $\X^1, \ldots \X^q$, with $|\X^i|=k$ for all $i= 1, \ldots q$. Thus $n=kq$. Afterwards one computes the score
$$
s(\X^i, x^0)= \sum_{\mu: x^\mu \in \X^i}\sum_{l,m=1}^d x_l^0 x_m^0 x_l^\mu x_m^\mu
$$
for each of these classes and takes the class $\X^i$ with the highest score. Then, on this class one performs exhaustive search. We will analyze under which condition on the parameters this algorithm works well and when it is effective.
To this end, we will show the following theorem.

\begin{theorem}\label{theo_sparse}
Assume that the query pattern is equal to one of the patterns in the database. The algorithm described above works asymptotically correctly, i.e. with an error probability that converges to 0 and is efficient, i.e. requires less computations than exhaustive search, if
\begin{enumerate}
\item  $d \ll k \ll d^2$, i.e. $\frac k{d^2} \to 0$ and $\frac k{d} \to \infty$,
\item and $q e^{-\frac 1 {32} \frac{d^2}k} \to 0$.
\end{enumerate}
\end{theorem}

\begin{proof}
To facilitate notation assume that $x^0=x^1$ (which is of course unknown to the algorithm), that $x^1 \in \X^1$,  and that $x_i^1= 1$ for the first $\tc$ indices $i$ and that $x_i^1= 0$ for the others. Note that for any $\vep>0$, with high probability for $c$ and $d$ large, we are able to assume that $\frac{\tc} d \in \left(\frac c d(1-\vep), \frac cd (1+\vep)\right)$.
Then for any $i$ we have
$
s(\X^i, x^0)= \sum_{\mu: x^\mu \in \X^i}\sum_{l,m=1}^\tc x_l^\mu x_m^\mu
$
and in particular
$$
s(\X^1, x^0)= \tc^2+\sum_{\mu: x^\mu \in \X^1\atop \mu \neq 1}\sum_{l,m=1}^\tc x_l^\mu x_m^\mu.
$$
We now want to show that for a certain range of parameters $n,k,$ and $d$ this algorithm works reliably, i.e. we want to show that
$$
\P(\exists i \ge 2: s(\X^i, x^0)\ge s(\X^1, x^0)) \to 0 \qquad \mbox{as}\quad   d \to \infty.
$$
Now since, $c$ and $\tc$ are close together when divided by $d$ we will replace $\tc$ by $c$ everywhere without changing the proof.
Trivially,
$$
\P(\exists i \ge 2: s(\X^i, x^0)\ge s(\X^1, x^0)) \le q
\P(s(\X^2, x^0)\ge s(\X^1, x^0))
$$
and we just need to bound the probability on the right hand side. Taking $X_\mu$ to be i.i.d. $B(c,\frac cd)$-distributed random variables, we may rewrite the probability on the right hand side as
$$
\P(s(\X^2, x^0)\ge s(\X^1, x^0))= \P(\sum_{\mu=1}^k X_\mu^2-\sum_{\mu=k+1}^{2k-1} X_\mu^2 \ge c^2).
$$
Centering the variables we obtain
\begin{eqnarray*}
&&\P(\sum_{\mu=1}^k X_\mu^2-\sum_{\mu=k+1}^{2k-1} X_\mu^2 \ge c^2)\\
&=&
\P(\sum_{\mu=1}^k (X_\mu-\frac{c^2}d)^2-\sum_{\mu=k+1}^{2k-1} (X_\mu-\frac{c^2}d)^2+\\&&
2 \frac{c^2}d(\sum_{\mu=1}^{k} X_\mu-\sum_{\mu=k+1}^{2k-1} X_\mu)-\frac{c^4}{d^2} \ge c^2)
\end{eqnarray*}
Now as $c \ll d$ the term $\frac{c^4}{d^2}$ is negligible with respect to $c^2$ and therefore up to terms of vanishing order
\begin{eqnarray*}
&&\P(\sum_{\mu=1}^k X_\mu^2-\sum_{\mu=k+1}^{2k-1} X_\mu^2 \ge c^2)\\
&\le&
\P(\sum_{\mu=1}^k (X_\mu-\frac{c^2}d)^2-\sum_{\mu=k+1}^{2k-1} (X_\mu-\frac{c^2}d)^2 \ge \frac {c^2}2)+\\&&
\P(\sum_{\mu=1}^{k} X_\mu-\sum_{\mu=k+1}^{2k-1} X_\mu \ge \frac d4)
\end{eqnarray*}
Let us treat the last term first. For $t>0$
\begin{eqnarray*}
&&\P(\sum_{\mu=1}^{k} X_\mu-\sum_{\mu=k+1}^{2k-1} X_\mu \ge \frac d4)\\&\le&
e^{-t\frac d4}(\E e^{tX_1})^k(\E e^{-tX_1})^{k-1}\\
&=& e^{-t\frac d4}\left(1+\frac cd (e^t-1)\right)^{kc}\left(1+\frac cd (e^{-t}-1)\right)^{(k-1)c}\\
&=& e^{-t\frac d4}\left(1+\frac cd (t+\frac {t^2} 2+\frac {t^3}6+\mathcal{O}(t^4))\right)^{kc}\\&&
\;\left(1+\frac cd (-t+\frac {t^2} 2-\frac {t^3}6+\mathcal{O}(t^4))\right)^{(k-1)c}\\
&\le&  e^{-t\frac d4}\exp\left(\frac {c^2}d t+k\frac{c^2}d t^2+ \frac {c^2}{3d} t^3+\mathcal{O}(\frac{c^2}d k t^4)\right)
\end{eqnarray*}
Now we take $t=\vep_d \sqrt[4]{\frac{d}{kc^2}}$, which implies that
\begin{eqnarray*}
&&\P(\sum_{\mu=1}^{k} X_\mu-\sum_{\mu=k+1}^{2k-1} X_\mu \ge \frac d4)\lesssim e^{-\mathrm{const.} \vep_d\sqrt[4]{\frac{d^{17}}{kc^2}}}
\end{eqnarray*}
if we take $\vep_d \to 0$ for $ d\to \infty$. This basically always works, if $k \le d^{15}$.

The other summand is slightly more complicated: Note that $\V (X_\mu-\frac{c^2}d)\approx \frac {c^2}d$; thus we compute
for $t>0$
\begin{eqnarray*}
&&\P(\sum_{\mu=1}^k (X_\mu-\frac{c^2}d)^2-\sum_{\mu=k+1}^{2k-1} (X_\mu-\frac{c^2}d)^2 \ge \frac {c^2}2)\\
&=&
\P\left(\sum_{\mu=1}^k \left(\frac{X_\mu-\frac{c^2}d}{\sqrt{\frac {c^2}d}}\right)^2-\sum_{\mu=k+1}^{2k-1} \left(\frac{X_\mu-\frac{c^2}d}{\sqrt{\frac{c^2}d}}\right)^2 \ge \frac d 2\right)\\
&\le & e^{-t\frac d 2} \left(\E e^{t \left(\frac{X_\mu-\frac{c^2}d}{\sqrt{\frac{c^2}d}}\right)^2}\right)^k
 \left(\E e^{-t \left(\frac{X_\mu-\frac{c^2}d}{\sqrt{\frac{c^2}d}}\right)^2}\right)^{k-1}
\end{eqnarray*}
where on the right hand side we can take any fixed $\mu$.

To estimate the expectations we will apply a version of Lindeberg's replacement trick \cite{Lindeberg:1920} or \cite{EiLo_linde}. To this end assume that $S_c:=\frac{X_\mu-\frac{c^2}d}{\sqrt{\frac {c^2}d}}\stackrel{d}{=}\sigma_1+\ldots+ \sigma_c$, for appropriately centered i.i.d. Bernoulli random variables $\sigma_k$ with variance 1. Moreover, let
$\xi = \frac 1 {\sqrt c} \sum_{i=1}^c \xi_i$, where the $(\xi_i)$ are i.i.d. standard Gaussian random variables. Finally set
$$
T_k := \frac 1 {\sqrt c}(\sigma_1 + \cdots + \sigma_{k-1} + \xi_{k+1} + \cdots + \xi_c)
$$
and $f(x)=e^{t x^2}$. Then we obtain by Taylor expansion
\begin{eqnarray*}
&&\left| \E \left( f\left(S_c\right) - f(\xi) \right) \right| \leq \\&& \sum_{k=1}^c \bigg|
 \E \biggl[ f \bigl(  \frac{1}{\sqrt{c}} (T_k + \sigma_k) \bigr) - f \bigl(\frac{1}{{\sqrt c}} (T_k + \xi_k) \bigr) \biggr]\biggr| \\
 & \le  &\sum_{k=1}^c \bigg| \E \biggl[ f' \bigl( \frac{T_k}{\sqrt{c}} \bigr) \, \frac{1}{\sqrt{c}} (\sigma_k - \xi_k)  + \frac 12 f'' \bigl( \frac{T_k}{\sqrt{c}} \bigr) \, \frac 1c (\sigma_k^2 - \xi_k^2) \biggr] \bigg|\\
 &&+ \frac 1 {6c^{3/2}}\sum_{k=1}^c \bigg| \E\biggl[ f^{(3)}(\zeta_1) \, \sigma_k^3 - \E f^{(3)}(\zeta_2) \, \xi_k^3 \biggr] \bigg|,
 \end{eqnarray*}
where $\zeta_1$ and $\zeta_2$ are random variable that lie within the interval $[T_k, T_k +\sigma_k]$, and $[T_k, T_k +\xi_k]$, respectively (possibly with the right and left boundary interchanged, if $\sigma_k$ or $\xi_k$ are negative).
First of all observe that for each $k$ the random variable $T_k$ is independent of the $\xi_k$ and $\sigma_k$ and $\xi_k$ and $\sigma_k$ have matching first and second moments. Therefore
$$
\sum_{k=1}^c \bigg| \E \biggl[ f' \bigl( \frac{T_k}{\sqrt{c}} \bigr) \, \frac{1}{\sqrt{c}} (\sigma_k - \xi_k)  + \frac 12 f'' \bigl( \frac{T_k}{\sqrt{c}} \bigr) \, \frac 1c (\sigma_k^2 - \xi_k^2) \biggr] \bigg|=0.
$$
For the second term on the right hand side we compute
$$f^{(3)}(x)=(12t^2x+8t^3x^3)f(x).$$
Observe that
$ \E [(\xi)^\nu e^{t \xi^2}] < \infty$ for $\nu=0,1,2,3$ for $t$ small enough and the expectation $\E\bigl[ \left(\frac{X_\mu-\frac{c^2}d}{\sqrt{\frac {c^2}d}}\right)^\nu f\left(\frac{X_\mu-\frac{c^2}d}{\sqrt{\frac {c^2}d}}\right)\bigr] $ is uniformly bounded in the number of summands $c$ for $\nu=0,1,2,3$ and if $t$ is small enough.
Finally the $\zeta_1$ and $\zeta_2$ have the form $T_k+ \alpha_k \frac{\xi_k}{\sqrt c}+\beta_k \sigma_k$ for some coefficients $\alpha_k$ and $\beta_k$. Therefore $ \bigg| \E\biggl[ f^{(3)}(\zeta_1) \, \sigma_k^3 - \E f^{(3)}(\zeta_2) \, \xi_k^3 \biggr] \bigg|$ is of order $t^2$ and
$$
\frac 1 {6c^{3/2}}\sum_{k=1}^c \bigg| \E\biggl[ f^{(3)}(\zeta_1) \, \sigma_k^3 - \E f^{(3)}(\zeta_2) \, \xi_k^3 \biggr] \bigg|=
\mathcal{O}(\frac {t^2} {\sqrt c}).$$

Moreover, we have for $t$ smaller than $1/2$ by Gaussian integration
\begin{equation}\label{Gauss}
\E[  f(\xi)]=\E[e^{t \xi^2}]=\frac 1{\sqrt{1-2t}}.
\end{equation}
Therefore, writing
$$\E[ f(S_c)] \leq \bigl| \E[ f(\xi)]\bigr|+ \big|\E[f(S_c) - f(\xi)] \bigr|  ,$$
we obtain
\begin{eqnarray*}
 &&\left(\E [ e^{t \frac{X_\mu^2-\frac{c^2}d}{\sqrt{\frac{c^2}d}}}]\right)^k \le\\ &&e^{-\frac k2 \log(1-2t)} \biggl(1+e^{\frac 12 \log(1-2t)}\mathcal{O}(\frac {t^2} {\sqrt c})\biggr)^k.\\
\end{eqnarray*}
In a similar fashion we can also bound:
\begin{eqnarray*}
 &&\left(\E [ e^{-t \frac{X_\mu^2-\frac{c^2}d}{\sqrt{\frac{c^2}d}}}]\right)^{k-1} \le\\&& e^{-\frac {k-1}2 \log(1+2t)} \biggl(1+e^{\frac 12 \log(1+2t)}\mathcal{O}(\frac {t^2} {\sqrt c})\biggr)^{k-1},\\
\end{eqnarray*}
where now we make us of $\E[e^{-t \xi^2}]=\frac 1{\sqrt{1+2t}}$ instead of \eqref{Gauss}.

Hence altogether we obtain
\begin{eqnarray*}
&&\P(\sum_{\mu=1}^k (X_\mu-\frac{c^2}d)^2-\sum_{\mu=k+1}^{2k-1} (X_\mu-\frac{c^2}d)^2 \ge \frac {c^2}2)
\quad \displaystyle \le\\&&  e^{-t\frac d 2} e^{2kt^2+\mathcal{O}(t+\frac {kt^2} {\sqrt c}+ kt^4)}.
\end{eqnarray*}
 Note that we will always assume that $c \to \infty$ such that $\frac {kt^2} {\sqrt c}$ is negligible with respect to $kt^2$. Choosing $t= \frac {d}{8k}$, which in particular ensures that $t$ converges to $0$ for large dimensions and therefore especially that \eqref{Gauss} can be applied,
 we obtain that asymptotically
$$
\P(\sum_{\mu=1}^k (X_\mu-\frac{c^2}d)^2-\sum_{\mu=k+1}^{2k-1} (X_\mu-\frac{c^2}d)^2 \ge \frac {c^2}2)\le e^{-\frac 1{32} \frac {d^2}{k}}
$$
which goes to 0, whenever $d^2 \gg k$. Hence for $ d \ll k \ll d^2$ we obtain that the method works fine if
$q \ll e^{\frac 1{32} \frac {d^2}{k}}$, which is our last condition.
\end{proof}

In a very similar fashion we can treat the case of an input pattern that is a corrupted version of one of the vectors in the database.

\begin{corollary}\label{cor_sparse}
Assume that the input pattern $x^0$ has a macroscopic overlap with one of the patterns in the database, i.e. $\sum_{l=1}^d x_l^0 x_l^\mu=\alpha c$, for a $\alpha \in (0,1)$, and $x^0$ has $c$ entries equal to $1$ and $d-c$ indices equal to $0$. The algorithm described above works asymptotically correctly, i.e. with an error probability that converges to 0 and is efficient if
\begin{enumerate}
\item  $d \ll k \ll d^2$, i.e. $\frac k{d^2} \to 0$ and $\frac k{d} \to \infty$,
\item and $q e^{-\frac {\alpha^4} {32} \frac{d^2}k} \to 0$.
\end{enumerate}
\end{corollary}

\begin{proof}
The proof is basically a rerun of the proof of Theorem \ref{theo_sparse}. Again without loss of generality assume that $x^1$ is our target pattern in the database such that $\sum_{l=1}^d x_l^0 x_l^1=\alpha c$, that $x^1 \in \X^1$,  and that $x_i^0= 1$ for the first $c$ indices $i$ and that $x_i^0= 0$ for the others.  Then for any $i$ we have
$$
s(\X^1, x^0)= \alpha^2 c^2+\sum_{\mu: x^\mu \in \X^1\atop \mu \neq 1}\sum_{l,m=1}^c x_l^\mu x_m^\mu
$$
while $s(\X^i, x^0)$ is structurally the same as in the previous proof. 
Therefore, following the lines of the proof of Theorem \ref{theo_sparse} and using the notation introduced there we find that
\begin{eqnarray*}
&&\P(\exists i \ge 2: s(\X^i, x^0)\ge s(\X^1, x^0)) \\& \le& q
\P(s(\X^2, x^0)\ge s(\X^1, x^0))\\
&=&
q \P(\sum_{\mu=1}^k (X_\mu-\frac{c^2}d)^2-\sum_{\mu=k+1}^{2k-1} (X_\mu-\frac{c^2}d)^2 \ge \frac {\alpha^2 c^2}2)+\\&&
\; \; \; q\P(\sum_{\mu=1}^{k} X_\mu-\sum_{\mu=k+1}^{2k-1} X_\mu \ge \frac {\alpha^2 d}4)
\end{eqnarray*}
Again the second summand vanishes as long as $k \ll d^{15}$ while for the first summand we obtain exactly as in the previous proof
$$  
q\P(\sum_{\mu=1}^{k} X_\mu-\sum_{\mu=k+1}^{2k-1} X_\mu \ge \frac {\alpha^2 d}4) \le q  e^{-t \alpha^2 \frac d 2} e^{2kt^2+\mathcal{O}(t+\frac {kt^2} {\sqrt c}+ kt^4)}.
$$
Now we choose $t= \frac{\alpha^2 d}{8k}$ to conclude as in the proof of Theorem \ref{theo_sparse} that
$$
q\P(\sum_{\mu=1}^{k} X_\mu-\sum_{\mu=k+1}^{2k-1} X_\mu \ge \frac {\alpha^2 d}4) \le q  e^{- \frac{\alpha^4 d^2}{32k} }
$$
which converges to 0 by assumption.
\end{proof}

\begin{rem}
 Erased or extremely corrupted patterns (i.e. $\alpha \to 0$) can also be treated.
\end{rem}

\section{Dense, unbiased patterns}
\label{dense}

Contrary to the previous section we will now treat the situation where the patterns are not sparse and do not have a bias. This is best modeled by choosing $x^\mu \in \{-1,1\}^d$ for some (large) $d$. We will now assume that the $x^\mu, \mu=1, \ldots, n$ are i.i.d. and have i.i.d. coordinates such that
$$
\P( x_l^\mu= 1)= \frac 12 = \P(x_l^\mu=-1)
$$
for all $l= 1, \ldots ,  d$ and all $\mu=1, \ldots ,n$.
Again we will suppose that there is an $x^\mu$ such that $d_H(x^\mu, x^0) $ is macroscopically large. To avoid technical problems we start with the situation that $x^0=x^1$.
Again we will also partition the patterns into equal-sized classes $\X^1, \ldots \X^q$, with $|\X^i|=k$ for all $i= 1, \ldots q$ and compute
$$
s(\X^i, x^0)= \sum_{\mu: x^\mu \in \X^i}\sum_{l,m=1}^d x_l^0 x_m^0 x_l^\mu x_m^\mu
$$
to compute overlap of a class $i$ with the query pattern.
We will prove:

\begin{theorem}\label{theo_dense}
Assume that the query pattern is equal to one of the patterns in the database. The algorithm described above works asymptotically correctly and is efficient if
\begin{enumerate}
\item  $d \ll k \ll d^2$, i.e. $\frac k{d^2} \to 0$ and $\frac k{d} \to \infty$,
\item and $q e^{-\frac 1 {8} \frac{d^2}k} \to 0$ if $d^4 \ll k^3$,
\item and $q e^{- \frac{d^2}{k^\frac 54}} \to 0$ if $k \le C d^{\frac 43}$, for some $C>0$.
\end{enumerate}
\end{theorem}

\begin{proof}
Without loss of generality, we may again suppose that $x^0 =x^1$ and now that $x^0_l=x^1_l \equiv 1$ for all $l=1, \ldots , d$, since otherwise we can flip the spins of all coordinates and all images, where this is not the case (recall that the $x_l^\mu$ are all unbiased i.i.d.). Then the above expression simplifies to
$$
s(\X^i, x^0)= \sum_{\mu: x^\mu \in \X^i}\sum_{l,m=1}^d x_l^\mu x_m^\mu.
$$
Especially
$$
s(\X^1, x^0)= d^2+\sum_{\mu: x^\mu \in \X^1\atop \mu \neq 1}\sum_{l,m=1}^d x_l^\mu x_m^\mu.
$$
Again we want to know for which parameters
$$
\P(\exists i \ge 2: s(\X^i, x^0)\ge s(\X^1, x^0)) \to 0 \qquad \mbox{as}\quad   d \to \infty.
$$
Trivially,
$$
\P(\exists i \ge 2: s(\X^i, x^0)\ge s(\X^1, x^0)) \le q
\P(s(\X^2, x^0)\ge s(\X^1, x^0)).
$$
By the exponential Chebychev inequality we obtain for any~$t>0$
\begin{eqnarray*}
&&\P(s(\X^2, x^0)\ge s(\X^1, x^0)) \\&=& \P( \sum_{\mu: x^\mu \in \X^2}\sum_{l,m=1}^d x_l^\mu x_m^\mu-
\sum_{\nu: x^\nu \in \X^1\atop \nu \neq 1}\sum_{l,m=1}^d x_l^\nu x_m^\nu \ge d^2)\\
&\le& e^{-t d^2}(\E \exp(t \sum_{l,m=1}^d x_l^\mu x_m^\mu))^k
\\&&\; \; \; \; \; \; (\E\exp(-t  \sum_{l,m=1}^d x_l^\nu x_m^\nu))^{k-1}.
\end{eqnarray*}
To calculate the expectations, we introduce a standard normal random variable $y$ that is independent of all other random variables occuring in the computation and by Gaussian expectation and Fubini property
\begin{eqnarray*}
&&\E \exp(t \sum_{l,m=1}^d x_l^\mu x_m^\mu)\\ &=& \E \exp(t(\sum_{l=1}^d x_l^\mu)^2)= \E_x \E_y e^{\sqrt{2t}y (\sum_{l=1}^d x_l^\mu)}\\
&=& \E_y \E_x e^{\sqrt{2t}y (\sum_{l=1}^d x_l^\mu)}\\
&=& \E_y \prod_{l=1}^d \E_{x^\mu_l} e^{\sqrt{2t}y x_l^\mu}\\
&=& \E_y \prod_{l=1}^d \cosh(\sqrt{2t} y)\\
&\le & E_y e^{dt y^2} =\frac 1 {\sqrt{1-2td}}
\end{eqnarray*}
if $dt< \frac 12$. Here we used the well-known estimate $\cosh(x) \le e^{\frac{x^2}2}$ for all $x$. Thus
$$
(\E \exp(t \sum_{l,m=1}^d x_l^\mu x_m^\mu))^k \le \exp(-\frac k2 \log(1-2td)).
$$
On the other hand, similary to the above, again by introducing a standard Gaussian random variable $y$ we arrive at
\begin{eqnarray*}
\E \exp(-t \sum_{l,m=1}^d x_l^\mu x_m^\mu) &=&  \E_x \E_y e^{i\sqrt{2t}y (\sum_{l=1}^d x_l^\mu)}\\
&=& \E_y \E_x e^{i\sqrt{2t}y (\sum_{l=1}^d x_l^\mu)}\\
&=& \E_y [\cos(\sqrt{2t} y)^d]\\
\end{eqnarray*}
To compute the latter we write for some $\vep_d >0$, that converges to zero if $d \to \infty$ at a speed to be chosen later
\begin{eqnarray*}
&&E_y [\cos(\sqrt{2t} y)^d]\\ &=& \int_{y \le  \frac{\vep_d}{(t^2d)^{1/4}}}\cos(\sqrt{2t} y)^d d\P+\int_{y >  \frac{\vep_d}{(t^2d)^{1/4}}}\cos(\sqrt{2t} y)^d d\P\\
&\le & \int_{y \le  \frac{\vep_d}{(t^2d)^{1/4}}} (1-ty^2+\mathcal{O}(t^2 y^4))^d d\P+\P(y >  \frac{\vep_d}{(t^2d)^{1/4}})\\
&\le& \int_{y \le  \frac{\vep_d}{(t^2d)^{1/4}}} e^{-tdy^2+\mathcal{O}(({\vep_d})^4)}d\P+e^{-\frac{\vep^2_d}{2t \sqrt{d}}}\\
&\le& e^{\mathcal{O}(({\vep_d)}^4)}\int e^{-tdy^2}d\P+e^{-\frac{\vep^2_d}{2t \sqrt{d}}}\\
&\le& e^{\mathcal{O}(({\vep_d})^4)}\frac 1 {\sqrt{1+2td}} +e^{-\frac{\vep^2_d}{2t \sqrt{d}}},\\
\end{eqnarray*}
where we used there the well-known estimate $\P[y\ge u] \le e^{-\frac{u^2}2}$ for all $u\ge 0$ and
$\E[e^{-td y^2}]=\frac 1{\sqrt{1+2td}}.$
Thus
\begin{eqnarray*}
&&(\E \exp(-t \sum_{l,m=1}^d x_l^\mu x_m^\mu)) ^{k-1}\\&\le &(e^{\mathcal{O}((\vep_d)^4)}\frac 1 {\sqrt{1+2td}} +e^{-\frac{\vep^2_d}{2t \sqrt{d}}})^{k-1}\\
&\le& e^{\mathcal{O}(k(\vep_d)^4)}  e^{-\frac{k-1} 2 \log (1+2td)}
\\&& \; \; \; \; \; \; (1+e^{\mathcal{O}(k(\vep_d)^4)}  e^{\frac{1} 2 \log (1+2td)} e^{-\frac{\vep^2_d}{2t \sqrt{d}}})^{k-1}\\
&\le& e^{\mathcal{O}(k(\vep_d)^4)}  e^{-\frac{k-1} 2 \log (1+2td)}
\\&& \; \; \; \; \; \; \exp(k e^{\mathcal{O}(k(\vep_d)^4)}  e^{\frac{1} 2 \log (1+2td)} e^{-\frac{\vep^2_d}{2t \sqrt{d}}})
\end{eqnarray*}
 We first choose $\vep_d=k^{-1/4}{\tilde\vep}_d$ for some small ${\tilde\vep}_d$ converging to $0$ as $d$ goes to infinity, such that the first factor on the right hand side converges to 1. 
Then we choose $k$ and $d$ such that the third term on the right hand side converges to one. This is true if
 $$\frac{1}{2} \log (1+2td) -\frac{\vep^2_d}{2t \sqrt{d}} +\log(k)\rightarrow -\infty.$$
 Anticipating our choices of $t$ and $d$ such that $dt$ goes to 0, we get the condition
 $-\frac{k^{-1/4}{\tilde\vep}^2_d}{2t \sqrt{d}} +\log(k)\rightarrow -\infty$ as $d$ and $k$ go to $0$.
 
 Now, we always suppose that $d \ll k$.
With these choices we obtain that 
\begin{eqnarray*}
&&\P(s(\X^2, x^0)\ge s(\X^1, x^0))\\
&\le&e^{-t d^2} e^{-\frac{k-1} 2 \log (1+2dt)} e^{-\frac{k} 2 \log (1-2dt)}\\&\approx&
e^{-t d^2} e^{-\frac{k} 2 \log (1-4d^2t^2)}
\end{eqnarray*}
Now we differentiate two cases. If $d^4 \ll k^3$ (note that we always have $d \ll k$) we expand the logarithm to obtain
$$
\P(s(\X^2, x^0)\ge s(\X^1, x^0))  \le e^{-t d^2 +2kd^2 t^2 +\mathcal{O}(kt^4 d^4)}
$$
Choosing $t=\frac 1{4k}$ we see that the term $\mathcal{O}(kt^4 d^4)$  is in fact $o(1)$ and therefore
$$
\P(s(\X^2, x^0)\ge s(\X^1, x^0))  \le e^{-\frac{ d^2}{8k} }
$$
Thus, if $k \ll d^2$ and $q \ll e^{\frac{ d^2}{8k} }$ we obtain
\begin{eqnarray*}
\P(\exists i \ge 2 : s(\X^i, x^0)\ge s(\X^1, x^0))  \to 0.
\end{eqnarray*}
On the other hand, if $d \ll k \le  \mathcal{O}(d^{4/3})$ we choose $t=k^{-5/4}$. Then by the same expansion of the logarithm
$$
\P(s(\X^2, x^0)\ge s(\X^1, x^0))  \le \exp\left(-\frac{d^2}{k^{\frac 54}} +2\frac{d^2}{k^{\frac 32 } } +\mathcal{O}((\frac dk)^{4})\right)
$$
 Now the $-\frac{d^2}{k^{\frac 54}}$ will always dominate the  $\frac{d^2}{k^{\frac 32 } }$-term and clearly the $\mathcal{O}$-term is again $o(1)$.
Thus
$$
\P(s(\X^2, x^0)\ge s(\X^1, x^0))  \le \exp\left(-\frac{d^2}{k^{\frac 54}}\right)(1+o(1)).
$$
Since $k \le \mathcal{O}(d^{4/3})$ the exponent diverges and we see that as long as
$\log q \ll   \frac{d^2}{k^{\frac 54}}$ again
\begin{eqnarray*}
\P(\exists i\ge 2 : s(\X^i, x^0)\ge s(\X^1, x^0))  \to 0.
\end{eqnarray*}
We can easily check  that our first condition   $-\frac{k^{-1/4}{\tilde\vep}^2_d}{2t \sqrt{d}} +\log(k)\rightarrow -\infty$ is fulfilled in both cases~: $t=\frac 1{4k}$ and $t=k^{-5/4}$, if $d \ll k$.

\end{proof}

Again, in a similar way one can treat the case of an input pattern that is a corrupted version of one of the vectors in the database.

\begin{corollary}\label{cor_dense}
Assume that the input pattern $x^0$ has a macroscopic overlap with one of the patterns, say $x^1$, in the database, i.e. $\sum_{l=1}^d x_l^0 x_l^1=\alpha d$, for a $\alpha \in ( 0 ,1)$. The algorithm described above works asymptotically correctly, and is efficient if
\begin{enumerate}
\item  $d \ll k \ll d^2$, i.e. $\frac k{d^2} \to 0$ and $\frac k{d} \to \infty$,
\item and $q e^{-\frac 1 {8} \frac{\alpha^4 d^2}k} \to 0$ if $d^4 \ll k^3$,
\item and $q e^{- \frac{\alpha^4 d^2}{k^\frac 54}} \to 0$ if $k \le C d^{\frac 43}$ for some $C>0$.
\end{enumerate}
\end{corollary}


\begin{proof}
The proof follows the lines of the proof of Theorem \ref{theo_dense} by using the condition that $\sum_{l=1}^d x_l^0 x_l^1=\alpha d$.
\end{proof}

\begin{rem}\normalfont
One might wonder, in how far taking a higher power than two of $\sum_{l=1}^d x_l^0 x_l^1$ in the computation of the score function changes the results. So let us assume we take
\begin{equation}\label{nspin}
s(\X^i, x^0)= \sum_{\mu: x^\mu \in \X^i}\sum_{l_1, \ldots l_n}^d x_{l_1}^0 x_{l_1}^0 \cdots x_{l_n}^\mu x_{l_n}^\mu.
\end{equation}
Then, of course, in the setting of the proof of Theorem \ref{theo_dense} we obtain
$$
s(\X^i, x^0)= d^n+ \sum_{\mu: x^\mu \in \X^1 \atop \mu \neq 1}\sum_{l_1, \ldots l_n}^d x_{l_1}^0 x_{l_1}^0 \cdots x_{l_n}^\mu x_{l_n}^\mu
$$
so we gain in the exponent. However the probability that $s(\X^1, x^0)$ is smaller than $s(\X^1, x^0)$ is then much more difficult to estimate. As a matter of fact, none of the techniques used in Sections 2 and 3 works, because a higher power cannot be linearized by Gaussian integration on the one hand, and exponential of powers larger than two of of Gaussian random variables are not integrable. A possible strategy could include exponential bounds as in Proposition 3.2. in \cite{Newman}. From here one also learns that in the  Hopfield model with $N$ neurons and $n$-spin interaction ($n  \ge 3$) the storage capacity grows like $N^{p-1}$. By similarity this could lead to the conjecture that replacing the
score function by \eqref{nspin} leads to a class size of $k \ll d^n$. However, in this the computational complexity of our algorithm would also increase.
\end{rem}

\input{simulations1}

\section{Conclusion}

We introduced a technique to perform approximate nearest neighbor search using associative memories to eliminate most of the unpromising candidates. Considering independent and identically distributed binary random variables, we showed there exists some asymptotical regimes for which both the error probability tends to zero and the complexity of the process becomes negligible compared to an exhaustive search. We also ran experiments on synthetic and real data to emphasize the interest of the method on realistic application scenarios.

A particular interest of the method is its ability to cope with very high dimension vectors, and even to provide better performance as the dimension of data increases. When combined with reduction dimension methods, it arises interesting perspectives on how to perform approximate nearest neighbor search with limited complexity.

There are many open ideas on how to improve the method further, including and not limiting to using smart pooling to directly identify the nearest neighbor without need to perform an exhaustive search, cascading the process using hierarchical partitioning, or improving the allocation strategies.

\bibliographystyle{abbrv}

\bibliography{LiteraturDatenbank}

\end{document}

%% file: simulations1.tex
\section{Experiments}

In order to stress the performance of the proposed system when considering nonasymptotic parameters, we propose several experiments. In subsection~\ref{synthetic}, we analyze the performance when using synthetic data. In subsection~\ref{real}, we run simulations using standard off-the-shelf real data.

\subsection{Synthetic data}
\label{synthetic}

We first present results obtained using synthetic data. Unless specified otherwise, each drawn point is obtained using Monte-Carlo simulations with at least 100,000 independent tests.

\subsubsection{Sparse patterns}

Considering data to be drawn i.i.d. with:
\[
\P( x_i^\mu= 1)= \frac cd = 1- \P(x_i^\mu=0)
\]
recall that we have four free parameters: $c$ the number of 1s in patterns, $d$ the dimension of patterns, $k$ the number of pattern in each class and $q$ the number of classes.

Our first two experiments consist of varying $k$ then $q$ while the other parameters are fixed. We choose the parameters $d=128$ and $c=8$. Figure~\ref{sparseinfk} depicts the rate at which the highest score is not achieved by the class containing the query vector (we call it the error rate), as a function of $k$ and for $q=10$. This function is obviously increasing with $k$, and presents a high slope for small values of $k$, emphasizing the critical choice of $k$ for applications.

\begin{figure}[h]
  \begin{center}
    \includegraphics[width=0.6\textwidth]{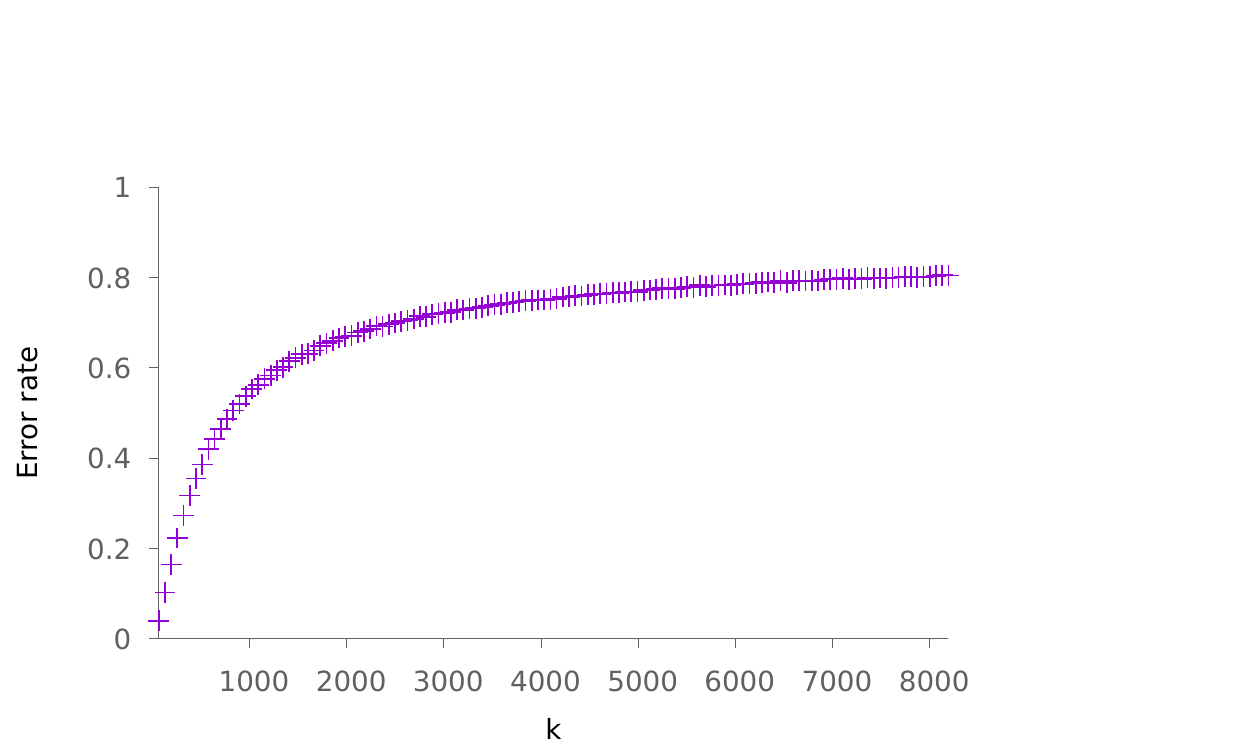}
  \end{center}
  \caption{Evolution of the error rate as a function of $k$. The other parameters are $q=10$, $d=128$ and $c=8$.}
  \label{sparseinfk}
\end{figure}

Figure~\ref{sparseinfq} depicts the probability of error as a function of $q$, for various choices of $k$. Interestingly, for reasonable values of $k$ the slope of the curve is not as dramatic as in Figure~\ref{sparseinfk}. When adjusting parameters, it seems thus more reasonable to increase the number of classes rather than the number of patterns in each class. This is not a surprising finding as the complexity of the process depends on $q$ but not on $k$.

\begin{figure}[h]
  \begin{center}
    \includegraphics[width=0.6\textwidth]{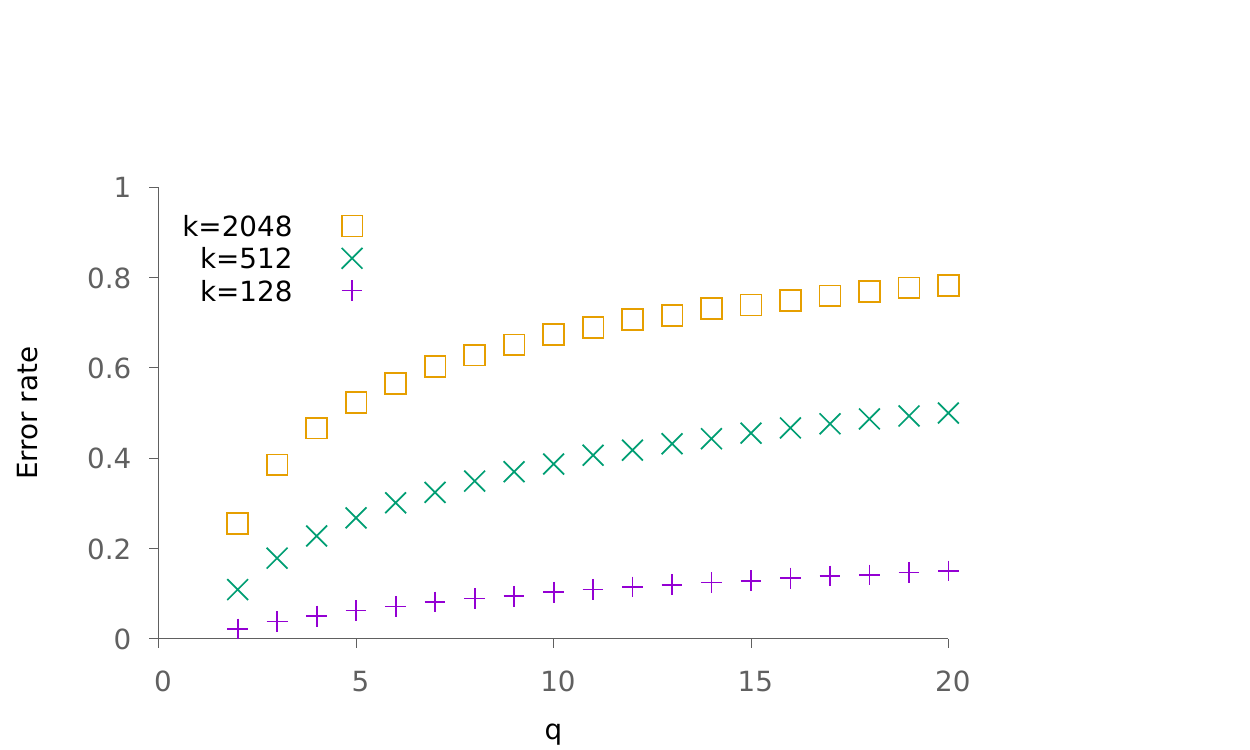}
  \end{center}
  \caption{Evolution of the error rate as a function of $q$ and for various values of $k$. The other parameters are $d=128$ and $c=8$.}
  \label{sparseinfq}
\end{figure}

A typical designing scenario would consist in, given data and its parameters $c$ and $d$, finding the best trade-off between $q$ and $k$ in order to split the data in the different classes. Figure~\ref{sparsedesign} depicts the error rate as a function of $k$ when $n = kq$ is fixed. To obtain these points we consider $c=8$, $d=128$, and $n = 16384$. There are two interesting facts that are pointed out in this figure. First we see that the trade-off is not simple as multiple local minima happens. This is not surprising as, with larger values of $k$, the decision becomes less precise and promissory. As a matter of fact, the last point in the curve only claims the correct answer is somewhere in a population of $8192$ possible solutions whereas the first point claims it is one of the $64$ possible ones. On the other hand, there is much more memory consumption for the first point where 256 square matrices of dimension $128$ are stored whereas only $2$ of them are used for the last point. Second, the error rate remains of the same order for all tested couples of values $k$ and $q$. This is an interesting finding as it emphasizes the design of a solution is more about complexity vs. precision of the answer -- in the sense of obtaining a reduced number of candidates -- than it is about error rate.

\begin{figure}[h]
  \begin{center}
    \includegraphics[width=0.6\textwidth]{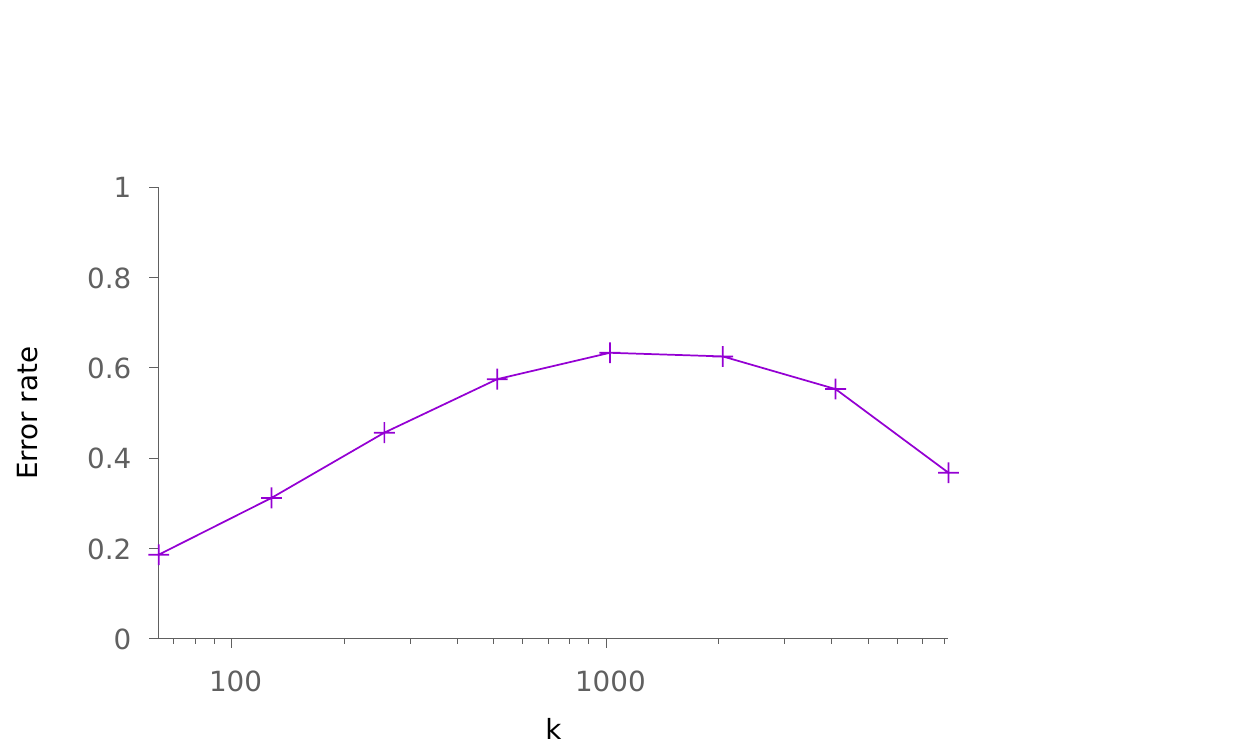}
  \end{center}
  \caption{Evolution of the error rate for a fixed number of stored messages $n = 16384$ as a function of $k$ (recall that $q = n / k$). The generated messages are such that $d=128$ and $c = 8$.}
  \label{sparsedesign}
\end{figure}

Finally, in order to evaluate both the tightness of the bounds obtained in Section~\ref{sparse} and the speed of convergence, we run an experiment in which $c = \log_2(d)$, $q=2$. We then depict the error rate as a function of $d$ and when $k= d^{1.5}$, $k=d^2$ and $k=d^{2.5}$. The result is depicted on Figure~\ref{sparseasym}. This figure supports the fact the obtained bound is tight, as illustrated by the curve corresponding to the case $k=d^2$ for which the error rate appears almost constant.

\begin{figure}[h]
  \begin{center}
    \includegraphics[width=0.6\textwidth]{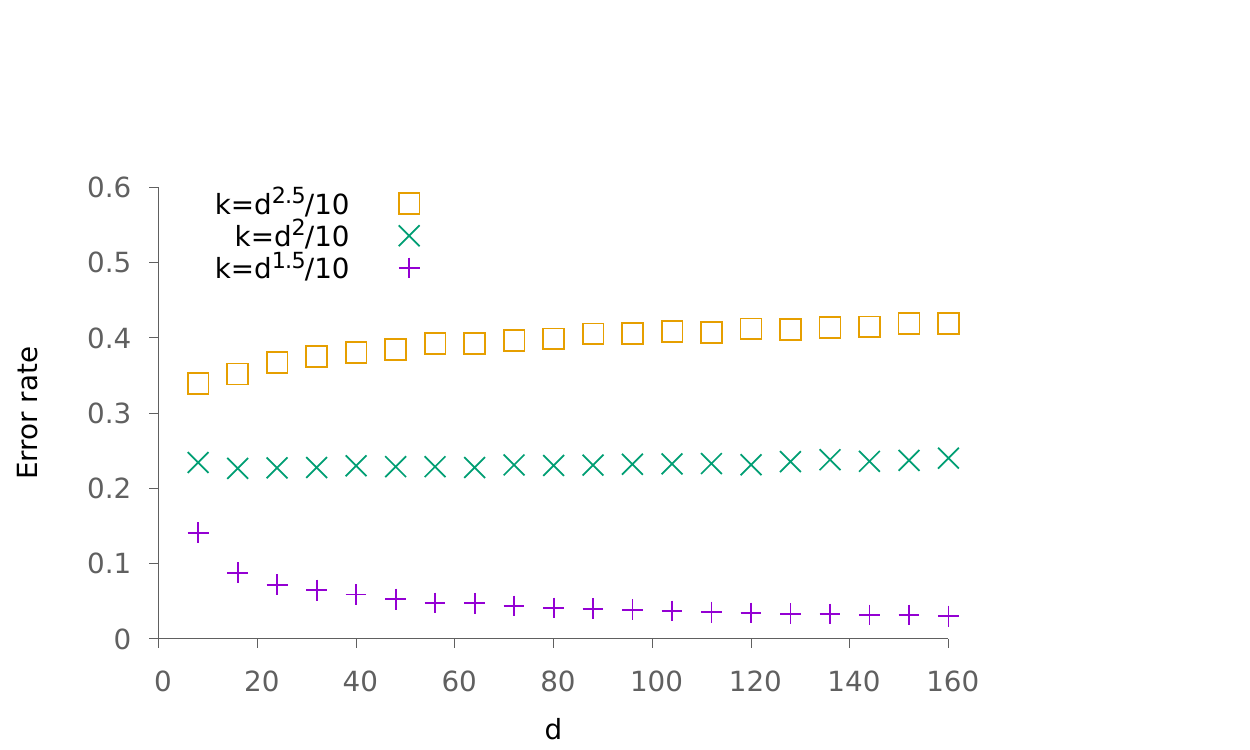}
  \end{center}
  \caption{Evolution of the error rate as a function of $d$. The other parameters are $q = 2$, $c = \log_2(d)$ and $k = d^\alpha/10$ with various values of $\alpha$.}
  \label{sparseasym}
\end{figure}

We ran the same experiments using cooccurence rules as initially proposed in~\cite{yu2015neural} (instead of adding contributions from distinct messages, we take the maximum). We observed small improvements in every case, even though they are not significant.

\subsubsection{Dense patterns}

Let us now consider data to be drawn i.i.d. according to the distribution :
\[
\P( x_i^\mu= 1)= \frac 12 = \P(x_i^\mu=-1)\;,
\]
recall that we then have three free parameters: $d$ the dimension of patterns, $k$ the number of patterns in each class and $q$ the number of classes.

Again, we begin our experiments by looking at the influence of $k$ (resp. $q$) to the performance. The obtained results are depicted in Figure~\ref{denseinfk} (resp. Figure~\ref{denseinfq}). To conduct these experiments, we have chosen $d=64$. The global slope resembles that of Figure~\ref{sparseinfk} and Figure~\ref{sparseinfq}.

\begin{figure}[h]
  \begin{center}
    \includegraphics[width=0.6\textwidth]{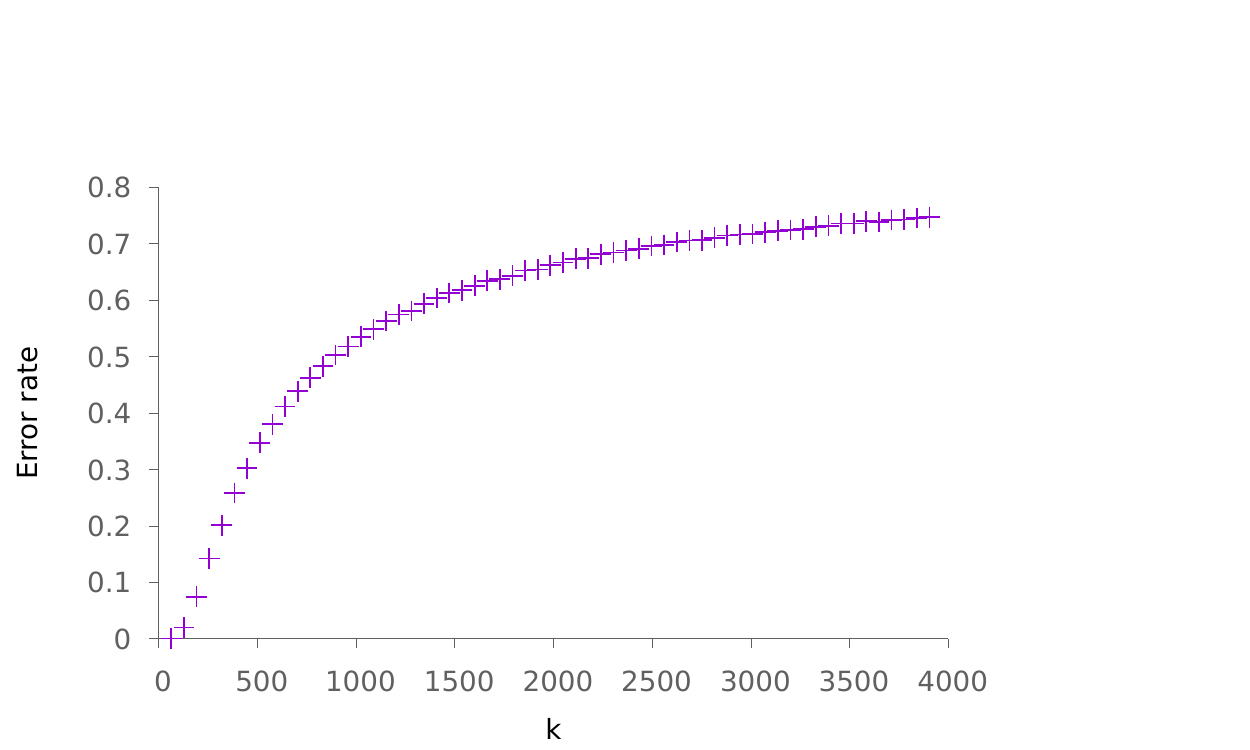}
  \end{center}
  \caption{Evolution of the error rate as a function of $k$. The other parameters are $q=10$ and $d=64$.}
  \label{denseinfk}
\end{figure}

\begin{figure}[h]
  \begin{center}
    \includegraphics[width=0.6\textwidth]{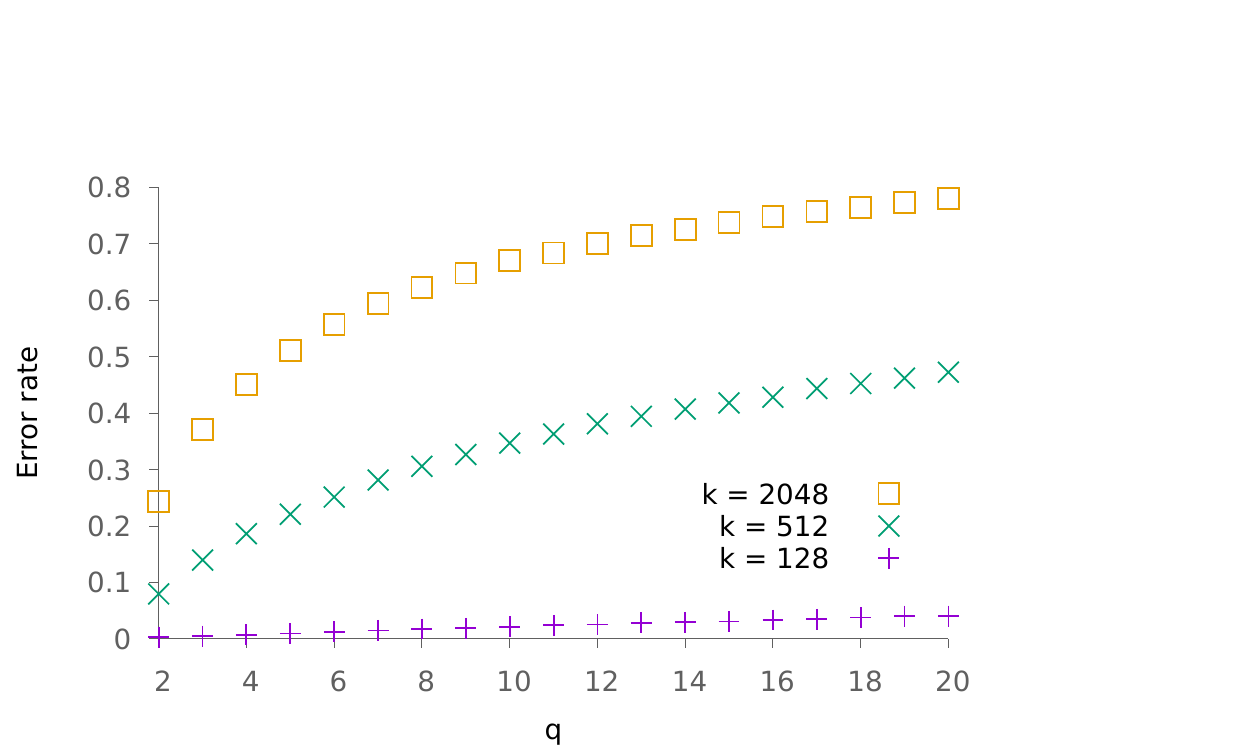}
  \end{center}
  \caption{Evolution of the error rate as a function of $q$. We fix the value $d=64$ and consider various values of $k$.}
  \label{denseinfq}
\end{figure}

Then, we consider the designing scenario where the number of samples is known. We plot the error rate depending on the choice of $k$ (and thus the choice of $q = n /k$). The results are depicted in Figure~\ref{densetradeoff}.

\begin{figure}[h]
  \begin{center}
    \includegraphics[width=0.6\textwidth]{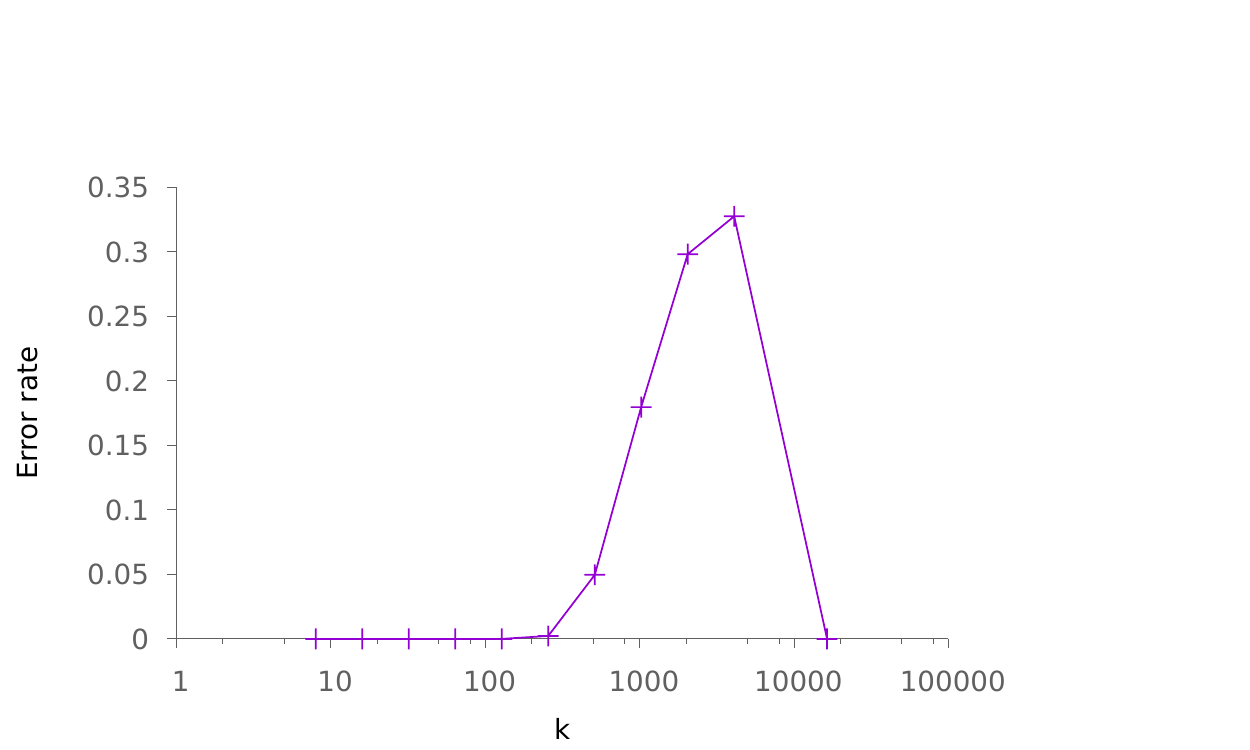}
  \end{center}
  \caption{Evolution of the error rate for a fixed total number of samples as a function of $k$. The other parameters are $n=16384$ and $d=64$.}
  \label{densetradeoff}
\end{figure}

Finally, we estimate the speed of convergence together with the tightness of the bound obtained in Section~\ref{dense}. To do so, we choose $k$ as a function of $d$ with $q=2$. The obtained results are depicted in Figure~\ref{denseasym}. Again, we observe that the case $k=d^2$ appears to be a limit case where error rate remains constant.

\begin{figure}[h]
  \begin{center}
    \includegraphics[width=0.6\textwidth]{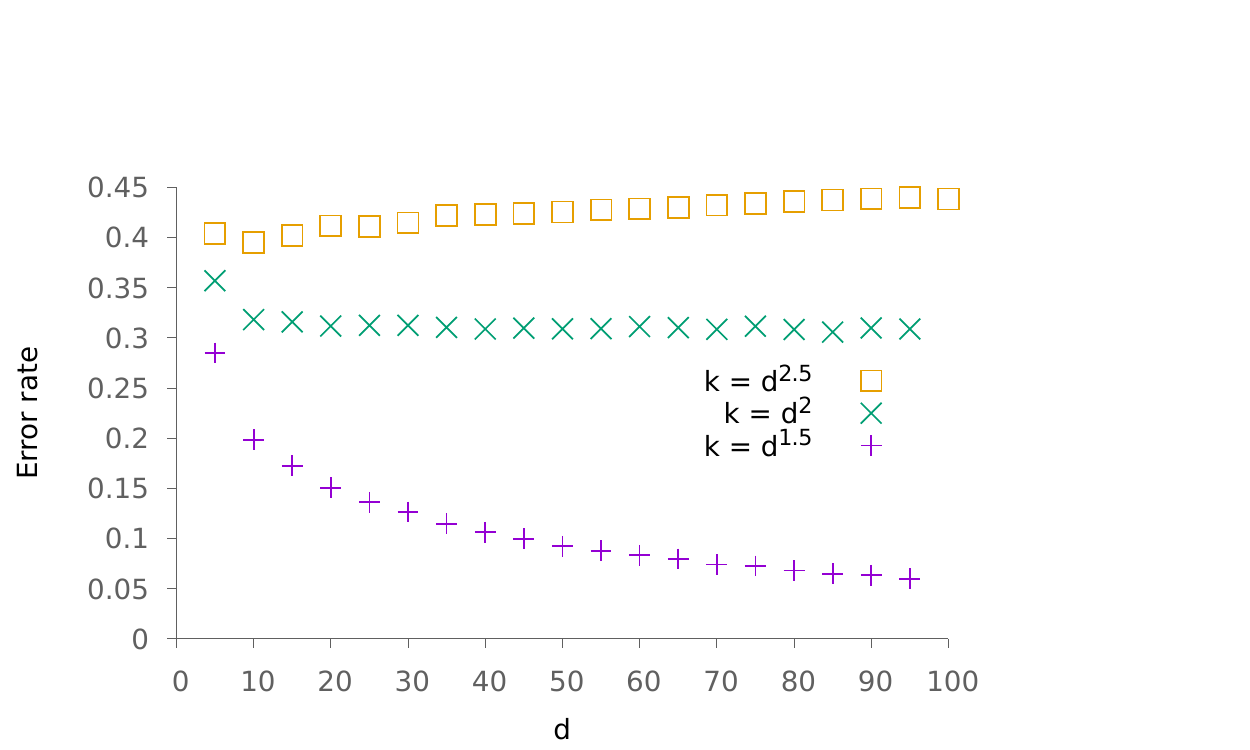}
  \end{center}
  \caption{Evolution of the error rate as a function of $d$. In this scenario, we choose $k = d^\alpha$ with various values of $\alpha$ and $q=2$.}
  \label{denseasym}
\end{figure}

\subsection{Real data}
\label{real}
In this section we use real data in order to stress the performance of the obtained methods on authentic scenarios.

Since we want to stress the interest of using our proposed method instead of classical exhaustive search, we are mainly interested in looking at the error as a function of the complexity. To do so, we modify our method as follows: we compute the scores of each class and order them from largest to smallest. We then look at the rate for which the nearest neighbor is in one of the first $p$ classes. With $p=1$, this method boils down to the previous experiments. Larger values of $p$ allows for more flexible trade-offs of errors versus performance. We call the obtained ratio the recall@1.

Considering $n$ vectors with dimension $d$, the computational complexity of an exhaustive search is $d n$ (or $c n$ for sparse vectors). On the other hand, the proposed method has a twofold computational cost: first the cost of computing each score, which is $d^2 q$ (or $c^2 q$ for sparse vectors), then the cost of exhaustively looking for the nearest neighbor in the selected $p$ classes, which is $p k d$ (or $p k c$ for sparse vectors). Note that the cost of ordering the $q$ obtained scores is negligible. In practice, distinct classes may have distinct number of elements, as explained later. In such case, complexity is estimated as an average of elementary operations (addition, multiplication, accessing a memory element) performed for each search.

When appliable, we also compare the performance of our method with Random Sampling (RS) to build a search tree, the methodology used by PySparNN\footnote{\url{https://github.com/facebookresearch/pysparnn}}, or Annoy\footnote{\url{https://github.com/spotify/annoy}}. In their method, random sampling of $r$ elements in the collection of vectors is performed, we call them anchor points. All elements from the collection are then attached to their nearest anchor point. When performing search, nearest anchor points are first searched then an exhaustive computation is performed with corresponding attached elements. Finally, we also look at the performance of a hybrid method in which associative memories are first used to identify which part of the collection should be investigated, then these parts are treated independently using the RS methodology.

There are several things that need to be changed in order to accommodate for real data. First, for nonsparse data we center data and then project each obtained vector on the hypersphere with radius 1. Then, due to the nonindependence of stored patterns, we choose an allocation strategy that greedily assign each vector to a class, rather than using a completely random allocation.

The allocation strategy we use consists of the following: each class is initialized with a random vector drawn without replacement. Then each remaining vector is assigned to the class that achieves the maximum normalized score. Scores are divided by the number of items $k$ currently contained in the class, as a normalization criterion. Note that more standard clustering techniques could be used instead.

The interest of this allocation strategy is emphasized in Figure~\ref{mnist}, where we use raw MNIST data to illustrate our point. MNIST data consists of grey-level images with 784 pixels representing handwritten digits. There are 60,000 reference images and 10,000 query ones. For various values of $k$, we compare the performance obtained with our proposed allocation and compare it with a completely random one. We also plot the results obtained using the RS methodology. We observe that in this scenario where the dimension of vectors is large with respect to their number, associative memories are less performant than their RS counterparts.

\begin{figure}[h]
  \begin{center}
    \includegraphics[width=0.6\textwidth]{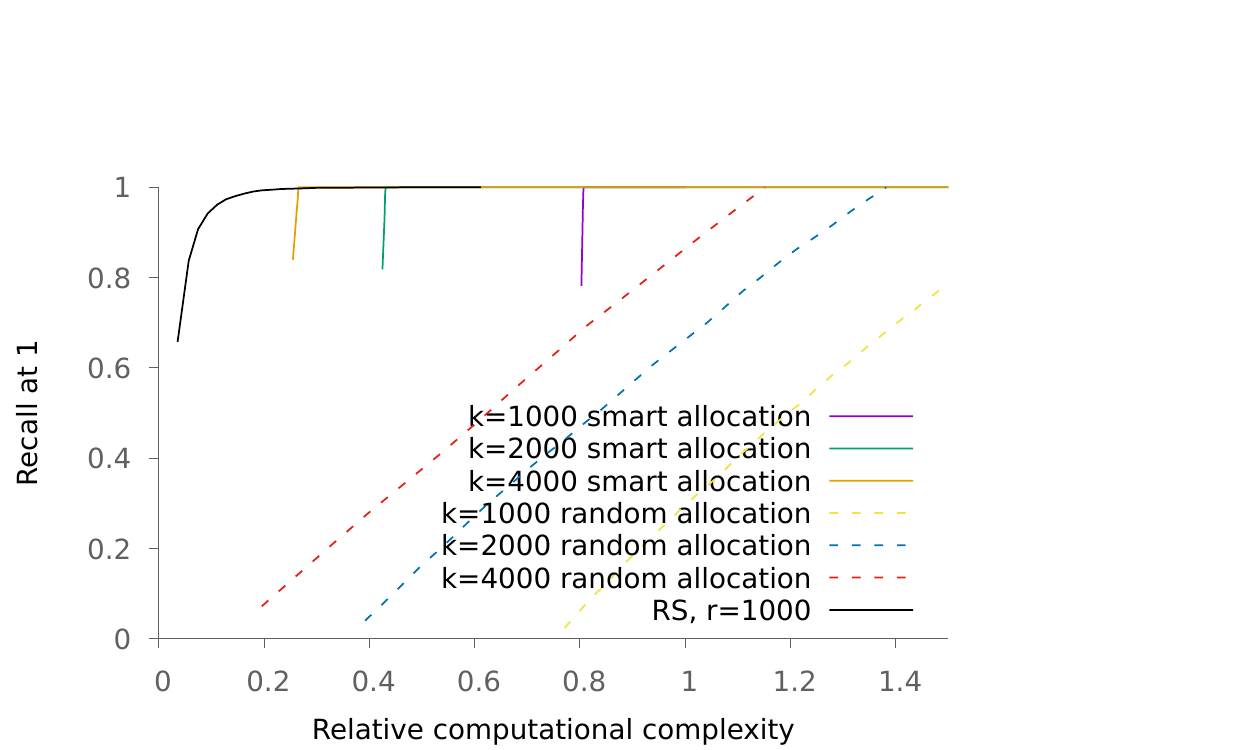}
  \end{center}
  \caption{Recall@1 on the MNIST dataset as a function of the relative complexity of the proposed method with regards to an exhaustive search, for various values of $k$ and allocation methods. Each curve is obtained by varying the value of $p$.}
  \label{mnist}
\end{figure}

We also run some experiments on a binary database. It consists of Santander customer satisfaction sheets associated with a Kaggle contest. There are 76,000 vectors with dimension 369 containing 33 nonzero values in average. In this first experiment, the vectors stored in the database are the ones used to also query it. The obtained results are depicted in Figure~\ref{kaggle}.

\begin{figure}[h]
  \begin{center}
    \includegraphics[width=0.6\textwidth]{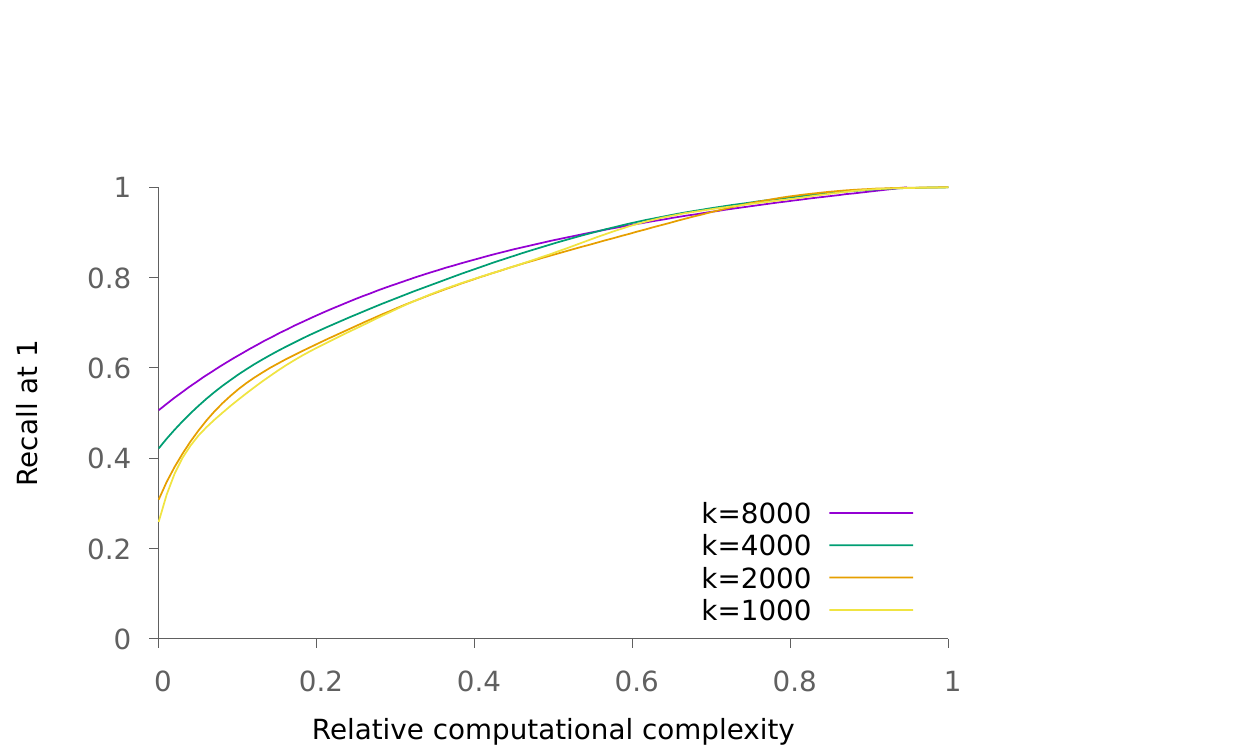}
  \end{center}
  \caption{Recall@1 on the Santander customer satisfaction dataset as a function of the relative complexity of the proposed method with regards to an exhaustive nearest neighbor search, for various values of $k$. Each curve is obtained by varying the value of $p$.}
  \label{kaggle}
\end{figure}

Then, we run experiments on the SIFT1M dataset\footnote{\url{http://corpus-texmex.irisa.fr/}}. This dataset contains 1 million 128-dimensions SIFT descriptors obtained using a Hessian-affine detector, plus 10,000 query ones.
The obtained results are depicted in Figure~\ref{SIFT1M}. As we can see here again the RS methodology is more efficient than the proposed one, but we managed to find hybrid parameters for which performance is improved.
To stress consistency of results, we also run experiments on the GIST1M dataset. It contains 1 million 960-dimensions GIST descriptors and 1,000 query ones. The obtained results are depicted in Figure~\ref{GIST1M}.

\begin{figure}[h]
  \begin{center}
    \includegraphics[width=0.6\textwidth]{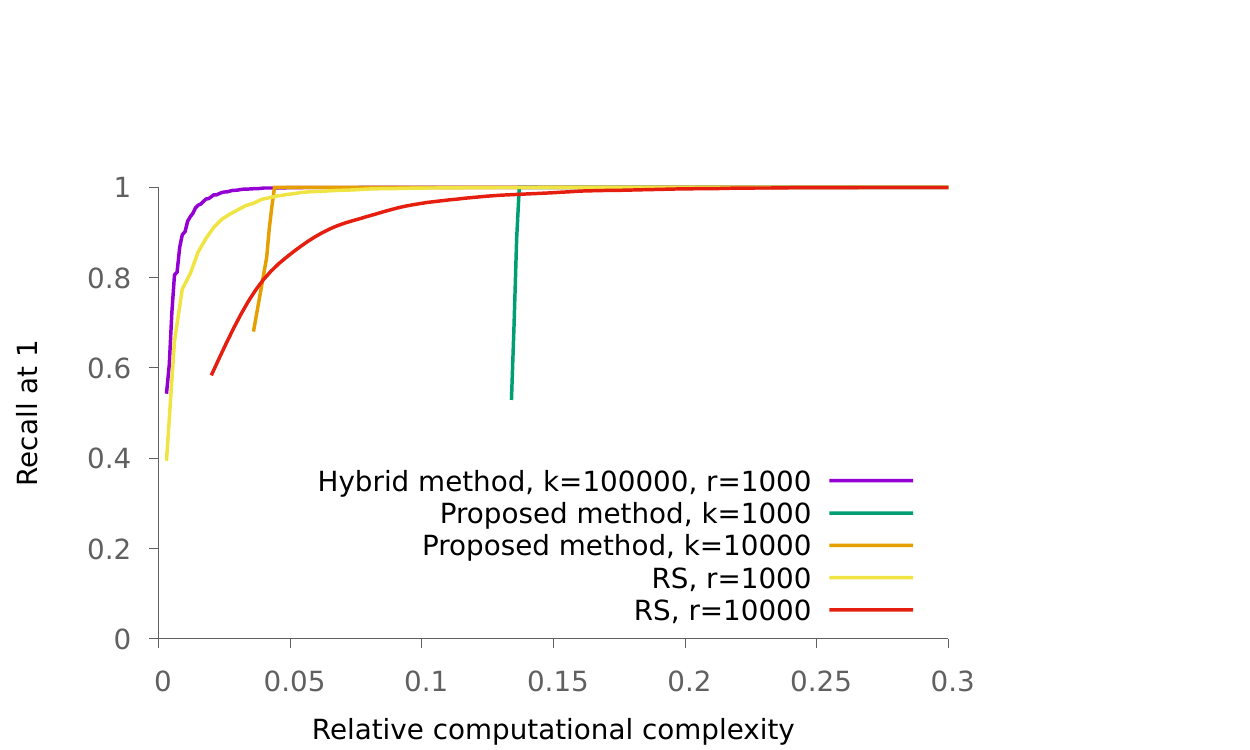}
  \end{center}
  \caption{Recall@1 on the SIFT1M dataset as a function of the relative complexity of the proposed method with regards to an exhaustive nearest neighbor search, for various values of $k$. Each curve is obtained by varying the value of $p$.}
  \label{SIFT1M}
\end{figure}

\begin{figure}[h]
  \begin{center}
    \includegraphics[width=0.6\textwidth]{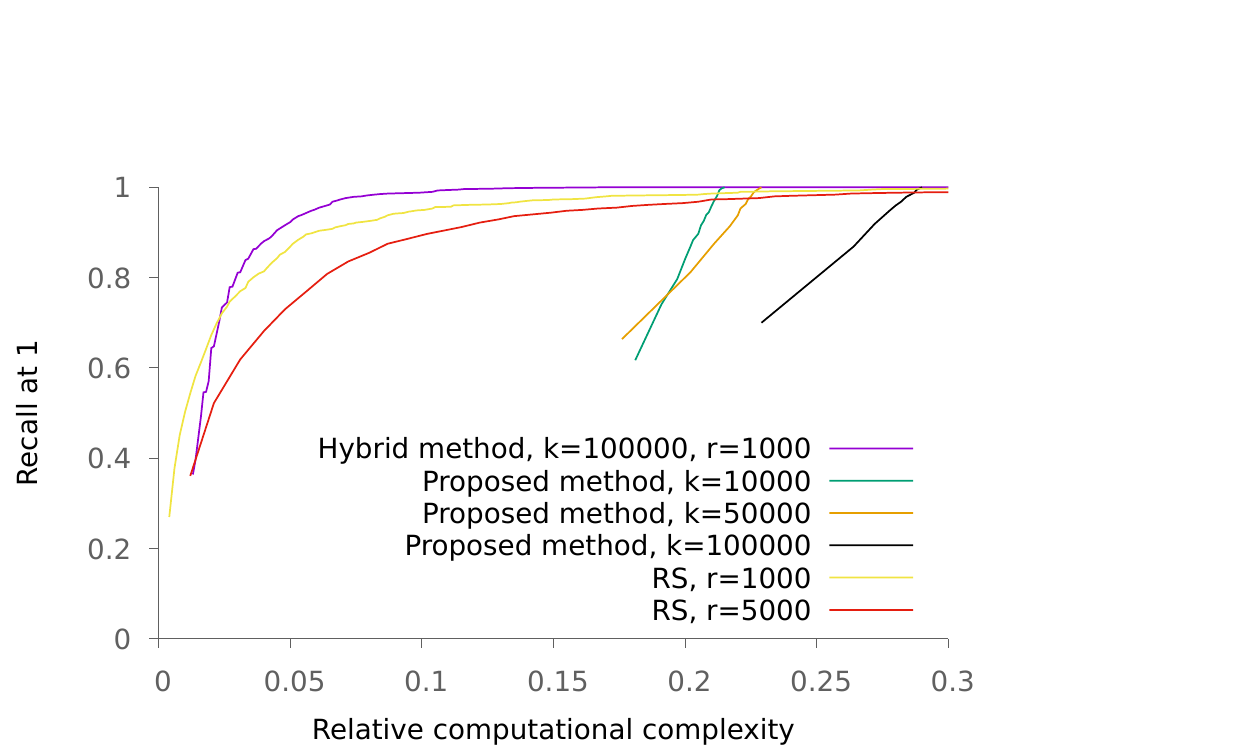}
  \end{center}
  \caption{Recall@1 on the GIST1M dataset as a function of the relative complexity of the proposed method with regards to an exhaustive nearest neighbor search, for various values of $k$. Each curve is obtained by varying the value of $p$.}
  \label{GIST1M}
\end{figure}